\documentclass[letterpaper]{article} 
\usepackage{aaai2026}  
\usepackage{times}  
\usepackage{helvet}  
\usepackage{courier}  
\usepackage[hyphens]{url}  
\usepackage{graphicx} 
\def\UrlFont{\rm}  
\usepackage{natbib}  
\usepackage{caption} 
\usepackage{algorithm}

\usepackage{newfloat}
\usepackage{listings}
\usepackage{array}
\usepackage[noend]{algpseudocode}
\usepackage{amsmath}
\usepackage{booktabs} 
\usepackage{amsthm}
\usepackage{amssymb}
\usepackage{cleveref}
\newtheorem{theorem}{Theorem}[section]
\newtheorem*{theorem*}{Theorem}

\DeclareCaptionStyle{ruled}{labelfont=normalfont,labelsep=colon,strut=off} 
\floatstyle{ruled}
\newfloat{listing}{tb}{lst}{}
\floatname{listing}{Listing}
\title{Sampling Control for Imbalanced Calibration in Semi-Supervised Learning}
\author {
    Senmao Tian,
    Xiang Wei,
    Shunli Zhang\thanks{Corresponding author.}
}
\affiliations {
    \textsuperscript{}School of Software Engineering, Beijing Jiaotong University\\
    smtian1204@gmail.com, slzhang@bjtu.edu.cn
}

\usepackage{bibentry}

\begin{document}

\maketitle

\begin{abstract}
Class imbalance remains a critical challenge in semi-supervised learning (SSL), especially when distributional mismatches between labeled and unlabeled data lead to biased classification. Although existing methods address this issue by adjusting logits based on the estimated class distribution of unlabeled data, they often handle model imbalance in a coarse-grained manner, conflating data imbalance with bias arising from varying class-specific learning difficulties. To address this issue, we propose a unified framework, SC-SSL, which suppresses model bias through decoupled sampling control. During training, we identify the key variables for sampling control under ideal conditions. By introducing a classifier with explicit expansion capability and adaptively adjusting sampling probabilities across different data distributions, SC-SSL mitigates feature-level imbalance for minority classes. In the inference phase, we further analyze the weight imbalance of the linear classifier and apply post-hoc sampling control with an optimization bias vector to directly calibrate the logits. Extensive experiments across various benchmark datasets and distribution settings validate the consistency and state-of-the-art performance of SC-SSL. 
\end{abstract}

\begin{links}
    \link{Code}{https://github.com/Sheldon04/SC-SSL}
\end{links}

\section{Introduction}
\label{sec:intro}

Semi-supervised learning (SSL) \cite{Berthelot2019MixMatchAH, Miyato2017VirtualAT, Tarvainen2017WeightaveragedCT} is a powerful strategy aimed at enhancing the generalization capabilities of deep neural networks (DNNs) \cite{chen2024adaptive, chen2024dual, hong2023hqretouch, lin2025nighthaze} by leveraging limited labeled data, particularly in scenarios where labeled samples are scarce. The core of most SSL methods \cite{Sohn2020FixMatchSS, Li2020CoMatchSL, Berthelot2019ReMixMatchSL} lies in generating pseudo-labels for unlabeled data and selecting reliable labels for model training. However, real-world data often exhibits a long-tailed distribution \cite{Li2022NestedCL, Wang2020LongtailedRB, Xiang2020LearningFM}, leading models to primarily focus on common categories during training, which results in an imbalance of pseudo-labels. This phenomenon has given rise to class-imbalanced semi-supervised learning (CISSL). Traditional CISSL approaches \cite{Fan2021CoSSLCO, Lee2021ABCAB, Wei2021CReSTAC} typically assume that the class distributions of labeled and unlabeled data are consistent, a premise that is often overly idealistic. In practical applications, class distributions may be inconsistent or even unknown \cite{Lai2022SmoothedAW, Kim2020DistributionAR, Oh2021DASODS}, especially when continuously collecting new data or processing data from different tasks. Variations in the distribution of unlabeled data significantly affect the performance of CISSL methods. To address these challenges, recent approaches \cite{Wang2022ImbalancedSL, Zheng2024BEMBA, Li2023TwiceCB, Lee2024CDMADCD} have focused on adapting to unknown and mismatched class distributions. Techniques such as ACR \cite{Wei2023TowardsRL}, CPE \cite{Ma2023ThreeHA}, and SimPro \cite{Du2024SimProAS} have been proposed to estimate the distance between distributions and train classifiers that adapt to specific or arbitrary distributions. However, due to the inherent limitations in the accuracy of classifiers, the adjusted pseudo-labels may not accurately reflect the true distribution of unlabeled data. Current methods \cite{Wei2023TowardsRL, Ma2023ThreeHA, Du2024SimProAS} improve model generalization by retaining a small portion of high-quality pseudo-labels to avoid confirmation bias \cite{Wang2022FreeMatchST}, resulting in a significant amount of unlabeled data not effectively contributing to training. Although some methods \cite{Guo2022ClassImbalancedSL, Yu2023InPLPT} use adaptive thresholds to adjust the class probabilities for pseudo-labeling, these methods rely on the model’s prediction confidence. These methods overlook the importance of addressing imbalance from a finer-grained perspective, which results in feature-level model bias and renders logit adjustment ineffective in targeting the root cause. As a result, these methods still yield suboptimal performance.

\begin{figure}[!t]
    \centering
    \includegraphics[width=7.5cm]{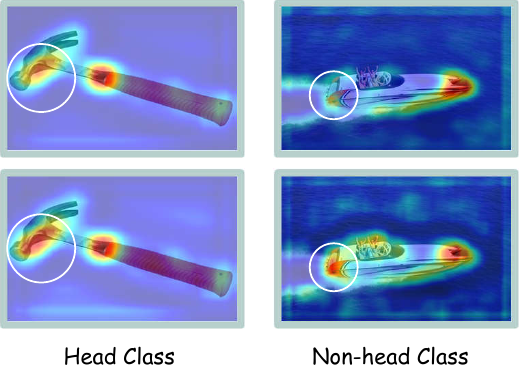}
    \setlength{\abovecaptionskip}{6pt}
    \setlength{\belowcaptionskip}{0pt}
    \caption{We select one image from a head class and one from a non-head class, and visualize their features using Grad-CAM \cite{Selvaraju2016GradCAMVE}. The top row shows the results from training with uniform sampling, while the bottom row shows the results from SC-SSL. It can be observed that the attention regions for the non-head class are significantly improved, capturing more relevant features. To enhance visual clarity, we applied upsampling.}
    \label{fig:viscam}
\end{figure}

To address this issue, we propose a novel method called SC-SSL (Sampling Control for Imbalanced Semi-supervised Learning), which suppresses model bias through decoupled sampling control.   Starting from known information, we divide the labeled data into head and non-head classes based on their sample sizes.   During training, instead of solely focusing on improving classification accuracy, SC-SSL emphasizes whether the learned features of non-head classes are balanced.

Without altering the training objective (i.e., the classification loss), we introduce an expansive classifier that focuses on feature learning for non-head classes.   We incorporate the expansion–separation assumption and leverage the inherent denoising property of consistency regularization, allowing even noisy pseudo-labels to be used effectively.   This balances the gradient contributions between head and non-head classes during training.
To improve the sampling probability of non-head classes, we theoretically identify key control factors and associate them with the expansion factor in the assumption.   This enables effective sampling regulation across various data distributions, allowing SC-SSL to mitigate the effects of data imbalance and suppress feature-level bias as \cref{fig:viscam} shows.

However, the supervised losses of the three classifiers are influenced by their respective sampling factors, leading to additional imbalance. This imbalance arises from both data distribution and optimization dynamics. As a result, an imbalance remains in the linear layers due to optimization.   At inference time, since feature imbalance has been alleviated, we can directly correct the optimization-induced imbalance in the classifier.
We observe that the bias term of the linear layer reflects a combination of optimization imbalance and data imbalance.   Under random sampling, head classes tend to have higher bias values;   under the expansive classifier's sampling, tail classes exhibit higher bias values;   whereas the output classifier—trained with balanced sampling for inference—has bias terms that exclude data imbalance, thus isolating the optimization bias.
Unlike the classifier weights, which require interaction with feature vectors and are harder to isolate, the bias term serves as a clean proxy for optimization imbalance.   Therefore, we treat it as an optimization bias vector to calibrate the final logits, producing an overall unbiased prediction.

Additionally, SC-SSL also leverages this bias estimation prior to training to approximate the distribution of the unlabeled data.   The resulting estimation is then used to initialize pseudo-label sampling strategies for both the balanced and expansive classifiers.

We summarize our contributions as follows: 
\begin{itemize}
\item We propose SC-SSL, a unified framework that mitigates both feature-level bias and logits-level bias in imbalanced settings through decoupled sampling control with an expansion classifier guided by the expansion–separation assumption.
\item We theoretically analyze the key factors for controlling pseudo-label sampling under class imbalance, and dynamically adjust the sampling probabilities during training.
\item Our proposed method achieves state-of-the-art performance on datasets such as CIFAR10-LT, CIFAR100-LT, STL10-LT, and ImageNet-127, across various distributions of unlabeled data.
\end{itemize}

\begin{figure*}[!t]
    \centering
    \includegraphics[width=15.0cm]{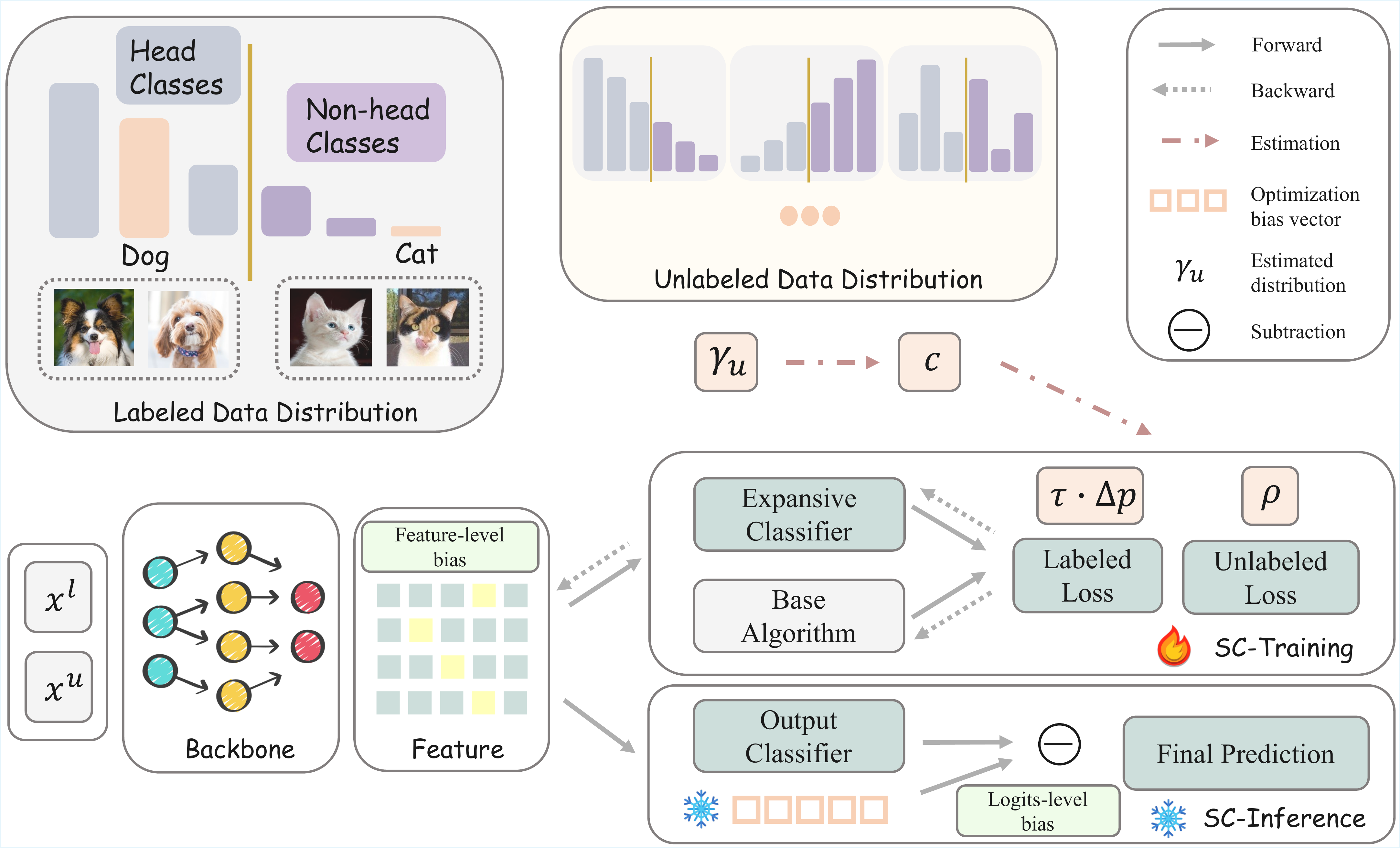}
    \setlength{\abovecaptionskip}{6pt}
    \setlength{\belowcaptionskip}{0pt}
    \caption{Illustration of the proposed framework. Details of the three key factors \(\gamma_u\), \(\Delta p\), and \(\rho\) can be found in \cref{eq:prob}. The utilization of these factors can be found in the Method section.}
    \label{fig:2}
\end{figure*}

\section{Related Work}
\label{sec:related}

\paragraph{Semi-supervised Learning} In the realm of SSL, a subset of algorithms has risen to prominence by leveraging unlabeled data to enhance model performance. These algorithms operate on the principle of generating pseudo-labels, thereby establishing a self-training mechanism that refines the model iteratively. Advanced SSL strategies have taken this a step further by merging pseudo labeling with consistency regularization. This fusion promotes consistent predictions for the same image across different presentations, enhancing the resilience of DNNs. Among these, FixMatch \cite{Sohn2020FixMatchSS} and ReMixMatch \cite{Berthelot2019ReMixMatchSL} stand out as particularly effective methods, achieving remarkable success in image recognition tasks and surpassing other SSL techniques in performance.

\paragraph{Class Imbalanced Semi-supervised Learning} CISSL has garnered considerable interest due to its effectiveness in various real-world applications. CReST \cite{Wei2021CReSTAC} emphasizes the importance of leveraging unlabeled instances from less common classes during its iterative self-training phase. Meanwhile, ABC \cite{Lee2021ABCAB} and CoSSL \cite{Fan2021CoSSLCO} introduce an additional classifier to prevent bias in the training process towards any specific class. Although these methods have shown significant performance enhancements, they generally operate under the assumption that the class distributions for both labeled and unlabeled data are the same. In real-world scenarios, the distribution of unlabeled data may be unknown and mismatched. To tackle the challenges of realistic CISSL, DASO \cite{Oh2021DASODS} innovatively adjusts the ratio of linear and semantic pseudo-labels based on the unknown class distribution of unlabeled data. Conversely, ACR \cite{Wei2023TowardsRL} seeks to improve consistency regularization by establishing predefined distribution anchors, yielding encouraging outcomes. Following ACR, SimPro \cite{Du2024SimProAS} takes a step further to model arbitrary unlabeled data distributions using probability theory. CPE \cite{Ma2023ThreeHA} takes a different approach by training multiple experts, with each expert focusing on modeling a specific distribution. When generating pseudo-labels, logit adjustments are made by ACR and SimPro based on the estimated distribution, which makes the quality of the pseudo-labels heavily dependent on the accuracy of the estimates. 

\section{Method}
\label{sec:method}

\subsection{Preliminaries}
\paragraph{Problem Setup} In the domain of CISSL, which is also referred to as long-tailed semi-supervised learning, we deal with a scenario that includes a labeled training dataset $\mathcal{D}_l = \{(x_i^{l},y_i^{l})\}_{i=1}^N$ and an unlabeled training dataset $\mathcal{D}_u=\{x_j^{u}\}_{j=1}^M$. Both datasets share the same set of classes. Each data point $x$ is associated with a class label $y_i\in\{0,1\}^K$. For the labeled dataset $\mathcal{D}_l$, the number of samples for the $k$-th class is denoted by $N_k$, and for the unlabeled dataset $\mathcal{D}_u$, it is denoted by $M_k$. The number of samples in the most frequent class in the labeled and unlabeled datasets are $N$ and $M$, respectively. Given that the distribution of classes in $\mathcal{D}_l$ is typically imbalanced, classes are arranged in descending order, resulting in $N_1 > N_2 > \ldots > N_K$. Here, the head classes are defined as those in the first half, while the non-head classes are those in the latter half. The imbalance ratio for the labeled dataset is defined as $\gamma_l = \frac{N_1}{N_K}$. The class distribution for the unlabeled dataset $\mathcal{D}_u$ remains unknown due to the absence of labels. Although we do not need to make assumptions about the distribution of unlabeled data, we still utilize predefined anchor distributions for a better comparison with previous methods. We consider two general types of distributions: long-tail and Gaussian.
For long-tail settings, we adopt consistent, uniform, and inverse distributions as the pre-defined anchors with an imbalance ratio \(\gamma_u\). In the inverse setting, the sample count for each class is sorted in ascending order, meaning it is the opposite of the labeled data. For Gaussian scenes, we have normal and inverse settings. The mean and variance of the Gaussian distribution are set to $(K-1)/2$ and $K/6$, respectively.

\paragraph{Notations}
$\mathcal{D}_u$ is the set of samples drawn from distribution $P$, reflecting the characteristics of the data in the input space $\mathcal{X}$.
We let \(B:\mathcal{X} \to \mathbb{R}^\mathcal{Q}\) denote a learned backbone mapping function (e.g., the continuous features output by a neural network), \(F:\mathbb{R}^\mathcal{Q} \to [K]\) the scoring function from features to logits, \(G:\mathcal{X} \to [K]\) the discrete labels from \(G(x) \triangleq argmax_i (\sigma  (F(B(x))_i )) \) and \(\sigma (\cdot)\) denotes the Softmax function. For different classification heads, \(F_b\) denotes the output classifier and \(F_e\) denotes the expansive classifier, respectively. They share the backbone \(B\).

\subsection{Motivation}
When the class distributions of labeled and unlabeled datasets are misaligned, many existing approaches~\cite{Wei2023TowardsRL, Du2024SimProAS, Ma2023ThreeHA} adopt conservative strategies, avoiding the use of pseudo-labels for tail classes due to concerns about noise and label errors. However, recent theoretical perspectives on self-training suggest that consistency regularization can still extract meaningful supervision from noisy pseudo-labels, as long as certain structural assumptions about the data distribution hold.

Specifically, the expansion assumption implies that even if only a small fraction of samples from a class are confidently predicted, their surrounding regions in the feature space are likely to contain samples from the same class. This implies that well-initialized class-conditional feature clusters can propagate label information to nearby unlabeled samples, thereby enhancing the balance of feature representations across classes. Formally, the class-conditional distribution \(P_i\) is said to satisfy an \((a, c)\)-expansion if for any subset \(V \subseteq \mathcal{X}\) with \(P_i(V) \leq a\), the following holds:
\begin{equation}
P_i(\mathcal{N}(V)) \geq \min\{c P_i(V), 1\},
\end{equation}
where \( \mathcal{N}(V) \) denotes a neighborhood around \(V\) in feature space.

Complementary to this, the separation assumption ensures local consistency under input perturbation. It requires that augmented versions of the same sample are predicted consistently:
\begin{align}
R_\mathcal{B}(G) & \triangleq \mathbb{E}_P\left[\mathbb{I}\left(\exists x^{\prime} \in \mathcal{A}(x) \text{ s.t. } G(x^{\prime}) \neq G(x)\right)\right] \notag \\
& \leq \mu,
\label{eq:sep}
\end{align}
where \( \mathcal{A}(x) \) defines the augmentation neighborhood of \(x\), and \( \mu \) reflects the classifier’s robustness to such perturbations.

Under these assumptions, the classifier's generalization error is upper-bounded by the quality of the pseudo-label generator and its robustness to augmentations. Specifically, the pseudo-label denoising theorem states that the error of the classifier \( \widehat{G} \) satisfies:
\begin{equation}
\mathrm{Err}(\widehat{G}) \leq \frac{2c}{c - 3} \mu,
\end{equation}
where \(c > 3\) denotes the expansion factor. The detailed derivation can be found in the Appendix of~\cite{Wei2020TheoreticalAO}. We adopt the FixMatch~\cite{Sohn2020FixMatchSS} setting, using \( \widehat{G} \) directly as the pseudo-labeler. This implies that even if pseudo-labels are noisy, as long as the model maintains low perturbation sensitivity (small \( \mu \)) and high feature-local consistency (large \( c \)), it can still facilitate class-balanced feature learning, especially for tail classes.

\subsection{The Design of SC-SSL}

\paragraph{Sampling Control for Training} 
In cases of long-tailed distributions of labeled data, the classifier's predictions tend to be biased towards the head classes, which leads to an imbalance in the predictions of pseudo-labels. This results in a misalignment between directly using thresholds to adjust sampling probabilities and prediction biases, leading to the occurrence of confirmation bias. Fortunately, prior work has proposed using logit adjustment-based balanced loss \cite{Menon2020LongtailLV, Wei2023TowardsRL, Du2024SimProAS, Ma2023ThreeHA} to reduce the bias of classifiers. Building on this insight and some explorations from previous studies \cite{Wang2022FreeMatchST, Yang2020RethinkingTV}, we utilize a simplified binary classification problem to further investigate the sampling probabilities of pseudo-labels. Given the imbalanced unlabeled data, we assume that the probability of the label \( Y \) being positive (+1) is \(\gamma \in \left(\frac{1}{2}, 1\right) \), while the probability of it being negative (-1) is \(1 - \gamma\). The input \( X \) has the following conditional distributions:
\[
X \mid Y = -1 \sim \mathcal{N}(\mu_1, \sigma_1^2), \quad X \mid Y = +1 \sim \mathcal{N}(\mu_2, \sigma_2^2).
\]
We assume \( \mu_2 > \mu_1 \) without loss of generality. Since we have already balanced the labeled loss and the adjusted input can be viewed as \( (\mathcal{x} + \Delta p)\) (\(\Delta p\) is set to the natural log of the label frequency of the positive class), we can roughly estimate the prediction of the pseudo-labels as \( (\mathcal{x} - \Delta p) \). Suppose the classifier outputs a confidence score defined as:
\(
s(\mathcal{x}) = 1 / (1 + \exp\left(-\beta\left( (\mathcal{x}-\Delta p) - \frac{\mu_1 + \mu_2}{2}\right)\right)),
\)
where \( \beta \) is a positive parameter reflecting the model's learning status and is expected to increase during training as the model becomes more confident gradually, and \( \frac{\mu_1 + \mu_2}{2} \) represents the Bayes' optimal linear decision boundary. We consider a scenario where a fixed threshold \( \rho \in \left(\frac{1}{2}, 1\right) \) is used to generate pseudo labels. A sample \( \mathcal{x} \) is assigned a pseudo label of +1 if \( s(\mathcal{x}) > \rho \) and -1 if \( s(\mathcal{x}) < 1 - \rho \). The pseudo label is set to 0 (masked) if \( 1 - \rho \leq s(\mathcal{x}) \leq \rho \). 
\begin{theorem}
\label{theorem:1}
The pseudo label \(Y_{psl}\) has the following probability distribution:
\begin{equation}\begin{aligned}
P(Y_{psl}=1)& =\gamma \Phi(\frac{\frac{\mu_{2}-\mu_{1}}{2}-\frac{1}{\beta}\log(\frac{\rho}{1-\rho}) -\Delta p}{\sigma_{2}}) \\ & \quad + (1-\gamma)\Phi(\frac{\frac{\mu_{1}-\mu_{2}}{2}-\frac{1}{\beta}\log(\frac{\rho}{1-\rho}) -\Delta p }{\sigma_{1}}), \\
P(Y_{psl}=-1)& =(1-\gamma)\Phi(\frac{\frac{\mu_2-\mu_1}2-\frac1\beta\log(\frac\rho{1-\rho})+\Delta p}{\sigma_1}) \\ & \quad +  \gamma\Phi(\frac{\frac{\mu_1-\mu_2}2-\frac1\beta\log(\frac\rho{1-\rho})+ \Delta p}{\sigma_2}), \\
P(Y_{psl}=0)& =1-P(Y_{psl}=1)-P(Y_{psl}=-1), 
\label{eq:prob}
\end{aligned}\end{equation}
where \(\Phi\) is the cumulative distribution function of a standard normal distribution.
\end{theorem}
The proof can be found in the Appendix. Then we can derive the following implications:
\begin{itemize}
\item The sampling probability of pseudo-labels is primarily influenced by \(\gamma\). Given that \( \mu_2 > \mu_1 \), the sampling probability for the head classes in the labeled data is naturally greater than that for the non-head classes, assuming other factors are not considered.

\item By controlling the logit adjustment amount \(\Delta p\), it is possible to mitigate or even reverse the extent to which the sampling probabilities are affected by the inherent imbalance in the data.

\item The choice of confidence also impacts the sampling probabilities; however, it is essential to consider that lower sampling probabilities can lead to confirmation bias as \(\beta\) gradually increases.
\end{itemize}

In dual-classifier settings~\cite{Lee2021ABCAB, Wei2023TowardsRL}, the original classifier is typically kept unchanged, while the second classifier is usually introduced to produce balanced output logits. However, this setup prevents us from adjusting the pseudo-label sampling probabilities of either classifier under the separation assumption.
On one hand, for the balanced classifier, balanced logits imply that the pseudo-labels sampled are likely to be balanced as well—meaning they approximately match the true distribution of pseudo-labels. In this true distribution, the proportion of non-head classes may be high or low. When the proportion is low, the sampling probability is essentially governed by the data imbalance factor \(\gamma\), meaning head classes are still sampled with high probability. Since this classifier is responsible for producing the final outputs and optimizing classification accuracy, it becomes difficult to make significant adjustments to \(\Delta p\), the change in sampling probability. As a result, increasing the sampling rate of non-head classes becomes challenging.
On the other hand, for the original classifier, since \(\Delta p\) cannot be modified directly, the only option is to increase the sampling probability of non-head classes by adjusting the confidence threshold \(\rho\). However, due to the classifier's inherent bias toward head classes, strong data augmentations tend to obscure class-discriminative features, causing the prediction \(G(x')\) to favor head classes. This violates the separation assumption defined in Eq.~\eqref{eq:sep}.

In summary, neither the output classifier nor the original classifier is suitable for significantly adjusting pseudo-label sampling. Therefore, we introduce an additional \textit{expansive classifier} that shares the same backbone but allows direct control over the sampling probability of non-head classes. 

\begin{figure}[!t]
    \centering
    \includegraphics[width=7.5cm]{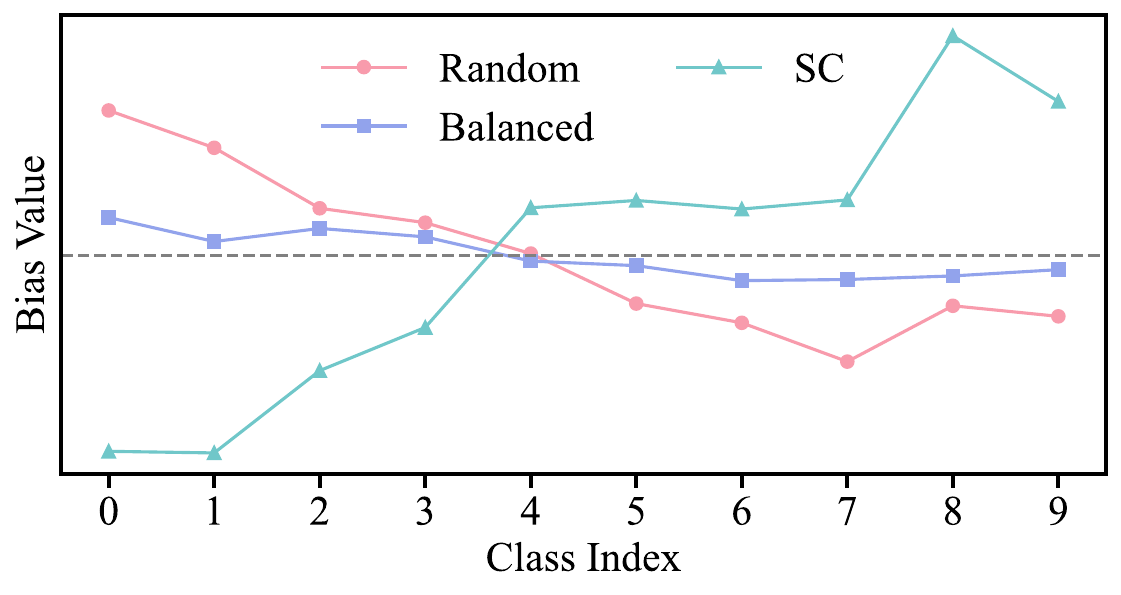}
    \setlength{\abovecaptionskip}{6pt}
    \setlength{\belowcaptionskip}{0pt}
    \caption{We choose the distribution setting of Consist to illustrate the pattern of the linear layer's bias term. Under random sampling, since the data follows a long-tailed distribution, the bias term tends to favor head classes. In contrast, under SC-SSL's controlled expansive sampling, the pseudo-label sampling probability is deliberately skewed toward minority classes, resulting in the opposite trend. Therefore, under balanced sampling, the bias term reflects the effect of optimization imbalance.}
    \label{fig:viscls}
\end{figure}

\paragraph{Sampling Control for Inference}
At inference time, despite the alleviation of feature imbalance through our expansive classifier during training, an optimization-induced imbalance still persists within the linear classifier parameters. Formally, a classifier can be expressed as:
\[
F(B(x)) = \mathbf{W} B(x) + \mathbf{b},
\]
where $\mathbf{W}$ represents the weight matrix, $B(x)$ is the feature vector produced by the backbone network, and $\mathbf{b}$ is the bias term. Since supervised losses of all linear classifiers are optimized using labeled data and traditional cross-entropy, the inherent class imbalance within labeled data inevitably impacts the bias term $\mathbf{b}$. We observe that these bias terms encapsulate two distinct sources of imbalance: data distribution imbalance and optimization-induced imbalance, as \cref{fig:viscls} shows. Under a random sampling regime, head classes typically exhibit higher bias values, reflecting data imbalance. Conversely, under expansive sampling during training, tail classes display elevated bias values, suggesting an overcorrection that inversely mirrors the original data imbalance.

However, the output classifier, trained using a balanced sampling strategy intended for inference, contains bias terms largely independent of data imbalance, thus isolating pure optimization-induced bias. Unlike classifier weights $\mathbf{W}$, which involve complex interactions with feature vectors and are challenging to decouple clearly, the bias term $\mathbf{b}$ provides a straightforward proxy for optimization imbalance.

Therefore, we define this bias term explicitly as an \textit{optimization bias vector}, denoted as $\mathbf{b}_{opt}$. At inference, we leverage this vector to calibrate final logits as follows:
\begin{equation}
\tilde{F}(B(x)) = {F_b}(B(x)) - \mathbf{b}_{opt} = W_b(B(x)),
\end{equation}
where $\tilde{F}(B(x))$ represents the bias-corrected logits. This simple yet effective adjustment ensures that the predictions are unbiased with respect to the optimization-induced imbalance, enhancing overall classification performance.

Additionally, we incorporate the $\mathbf{b}_{opt}$ as a prior to approximate the underlying distribution of unlabeled data. After several epochs of estimation training and performing inference on the unlabeled dataset, we can roughly determine the quantity of each class \(N^e = \{N^e_1, \ldots, N^e_k\}\) for \(k=1,..., K\) according to the adjusted balanced output. Let there be \(O\) predefined distributions \(P^{(1)}, P^{(2)}, \ldots, P^{(o)}\), where each distribution \(P^{(o)} = \{p_k^{(o)}\}_{k=1}^K\) satisfies \(\sum_{k=1}^K p_k^{(o)} = 1\) for \(o = 1, \ldots, O\). These distributions are then re-scaled based on the counts of \(N_e\):
\[
Q_k^{(o)} = \frac{p_k^{(o)} \cdot sum(N^e)}{\sum_{i=1}^K p_i^{(o)}}.
\]
The closest distribution can be identified as (\(D_{KL}\) measures KL divergence):
\[
o^* = \arg\min_{o} D_{KL}(N^e, Q^{(o)}).
\]

\begin{table*}[!t]
  \centering
  \small 
  \begin{tabular}{@{}c|cccc|cccc@{}}
    \toprule
    Dataset & \multicolumn{4}{c|}{CIFAR10-LT} & \multicolumn{4}{c}{CIFAR100-LT} \\
    \cmidrule(lr){1-1} \cmidrule(lr){2-5} \cmidrule(lr){6-9}
    Distribution & \multicolumn{2}{c}{Consist} & \multicolumn{2}{c|}{Inverse} & Consist & Consist & Inverse & Inverse \\
    \(\gamma_l \text{-} \gamma_u\) & \multicolumn{4}{c|}{100-100} & 10-10 & 15-15 & 10-10 & 15-15 \\
    \midrule
    \(N \text{-} M\) & 1500-3000 & 500-4000 & 1500-3000 & 500-4000 & \multicolumn{4}{c}{150-300} \\
    \midrule
    Supervised & 64.25 {\tiny (±0.43)} & 49.55 {\tiny (±0.92)} & 64.25 {\tiny (±0.43)} & 49.55 {\tiny (±0.92)} & 47.92 {\tiny (±0.55)} & 45.65 {\tiny (±0.63)} & 48.25 {\tiny (±0.39)} & 46.23 {\tiny (±0.50)} \\
    \midrule
    FixMatch+LA & 81.49 {\tiny (±0.77)} & 75.26 {\tiny (±1.58)} & 80.68 {\tiny (±0.54)} & 73.23 {\tiny (±1.01)} & 58.56 {\tiny (±0.19)} & 55.26 {\tiny (±0.05)} & 58.21 {\tiny (±0.24)} & 57.93 {\tiny (±0.69)} \\
    w/ DARP & 80.41 {\tiny (±0.63)} & 76.93 {\tiny (±0.52)} & 81.19 {\tiny (±0.37)} & 74.05 {\tiny (±1.28)} & 59.94 {\tiny (±0.33)} & 55.71 {\tiny (±0.42)} & 59.88 {\tiny (±0.39)} & 56.04 {\tiny (±0.25)} \\
    w/ CReST & 79.45 {\tiny (±0.15)} & 77.24 {\tiny (±0.35)} & 85.92 {\tiny (±0.76)} & 78.52 {\tiny (±0.90)} & 58.78 {\tiny (±0.64)} & 56.20 {\tiny (±0.13)} & 60.44 {\tiny (±0.09)} & 56.40 {\tiny (±0.17)} \\
    \midrule
    FixMatch+CDMAD & 83.33 {\tiny (±0.52)} & 79.10 {\tiny (±1.32)} & 77.76 {\tiny (±0.56)} & 74.65 {\tiny (±1.88)} & 59.31 {\tiny (±0.42)} & 56.05 {\tiny (±0.23)} & 58.72 {\tiny (±0.45)} & 55.46 {\tiny (±0.30)} \\
    w/ ABC & \underline{84.56} {\tiny (±0.12)} & 80.09 {\tiny (±0.31)} & 83.55 {\tiny (±0.08)} & 79.46 {\tiny (±0.29)} & 58.95 {\tiny (±0.33)} & 56.09 {\tiny (±0.26)} & 60.23 {\tiny (±0.18)} & 57.49 {\tiny (±0.04)} \\
    w/ DASO & 79.56 {\tiny (±0.98)} & 75.24 {\tiny (±0.66)} & 75.01 {\tiny (±0.12)} & 68.88 {\tiny (±0.35)} & 58.72 {\tiny (±0.14)} & 56.39 {\tiny (±0.28)} & 61.06 {\tiny (±0.29)} & 58.12 {\tiny (±0.17)} \\
    \midrule
    FixMatch & 76.49 {\tiny (±0.72)} & 73.14 {\tiny (±0.13)} & 68.92 {\tiny (±0.79)} & 62.52 {\tiny (±0.93)} & 57.61 {\tiny (±0.44)} & 54.02 {\tiny (±0.14)} & 57.10 {\tiny (±0.31)} & 53.88 {\tiny (±0.25)} \\
    w/ ABC & 82.69 {\tiny (±0.64)} & 79.96 {\tiny (±0.09)} & 83.22 {\tiny (±0.47)} & 79.26 {\tiny (±0.35)} & 58.30 {\tiny (±0.24)} & 55.67 {\tiny (±0.09)} & 59.24 {\tiny (±0.40)} & 56.65 {\tiny (±0.33)} \\
    w/ DASO & 78.68 {\tiny (±0.59)} & 73.62 {\tiny (±0.40)} & 74.52 {\tiny (±0.61)} & 67.59 {\tiny (±1.50)} & 58.16 {\tiny (±0.45)} & 54.92 {\tiny (±0.10)} & 59.25 {\tiny (±0.19)} & 55.38 {\tiny (±0.14)} \\
    w/ ACR & 84.10 {\tiny (±0.39)} & \underline{81.52} {\tiny (±0.24)} & \underline{89.46} {\tiny (±0.42)} & \underline{84.88} {\tiny (±0.16)} & \underline{60.34} {\tiny (±0.66)} & \underline{57.46} {\tiny (±0.32)} & \underline{61.79} {\tiny (±0.43)} & 58.53 {\tiny (±0.51)} \\
    w/ CPE & 84.46 {\tiny (±0.20)} & 80.89 {\tiny (±0.09)} & 87.10 {\tiny (±0.21)} & 83.76 {\tiny (±0.32)} & 59.83 {\tiny (±0.29)} & 57.00 {\tiny (±0.51)} & 60.83 {\tiny (±0.30)} & \underline{58.54} {\tiny (±0.08)} \\
    w/ SC-SSL (Ours) & \textbf{86.53} {\tiny (±0.16)} & \textbf{83.89} {\tiny (±0.35)} & \textbf{89.97} {\tiny (±0.20)} & \textbf{86.02} {\tiny (±0.09)} & \textbf{60.65} {\tiny (±0.14)} & \textbf{57.88} {\tiny (±0.26)} & \textbf{62.99} {\tiny (±0.42)} & \textbf{60.27} {\tiny (±0.14)} \\
    \bottomrule
  \end{tabular}
    \caption{Test accuracy of previous CISSL algorithms and our proposed SC-SSL under consistent and inverse distributions with different numbers of training samples on CIFAR10-LT and CIFAR100-LT benchmarks. The network architecture is WRN-28-2 trained from scratch. We highlight the best number in \textbf{bold} and the second best is \underline{underlined}.}
  \label{tab:1}
\end{table*}

\begin{table*}[!t]
  \centering
  \small 
  \begin{tabular}{@{}c|cccccc|cc@{}}
    \toprule
    Dataset & \multicolumn{6}{c|}{CIFAR10-LT} & \multicolumn{2}{c}{STL10-LT} \\
    \cmidrule(lr){1-1} \cmidrule(lr){2-7} \cmidrule(lr){8-9}
    Distribution & \multicolumn{4}{c}{Uniform} & Gaussian & Gaussian-I & Unknown & Unknown \\
    \(\gamma_l \text{-} \gamma_u\)  & \multicolumn{4}{c}{100-1} & 100- & 100- & 10- & 20- \\
    \midrule
    \(N \text{-} M\) & 1500-3000 & 1500-300 & 500-4000 & 500-400 & \multicolumn{2}{c}{1500-3000} & \multicolumn{2}{|c}{150-} \\
    \midrule
    FixMatch & 81.51 {\tiny (±1.15)} & 73.27 {\tiny (±0.99)} & 73.01 {\tiny (±3.81)} & 66.47 {\tiny (±0.84)} & 76.25 {\tiny (±1.29)} & 74.74 {\tiny (±2.03)} & 66.67 {\tiny (±1.89)} & 55.99 {\tiny (±3.84)} \\
    w/ ABC & 87.89 {\tiny (±1.24)} & 81.89 {\tiny (±0.32)} & 86.92 {\tiny (±0.68)} & 78.26 {\tiny (±0.60)} & 86.15 {\tiny (±0.45)} & 85.78 {\tiny (±0.39)} & 71.12 {\tiny (±1.36)} & 66.23 {\tiny (±2.14)} \\
    w/ CDMAD & 90.79 {\tiny (±0.43)} & 82.12 {\tiny (±0.74)} & 87.11 {\tiny (±0.73)} & 80.65 {\tiny (±0.78)} & 87.64 {\tiny (±0.36)} & 86.25 {\tiny (±0.67)} & 71.66 {\tiny (±0.22)} & 66.84 {\tiny (±0.37)} \\
    w/ ACR & 93.52 {\tiny (±0.11)} & 84.61 {\tiny (±0.50)} & 92.13 {\tiny (±0.15)} & 80.10 {\tiny (±1.21)} & \underline{89.65} {\tiny (±0.68)} & \underline{90.00} {\tiny (±0.76)} & \underline{76.94} {\tiny (±0.35)} & \underline{74.53} {\tiny (±0.83)} \\
    w/ CPE & \textbf{93.81} {\tiny (±0.14)} & \underline{85.86} {\tiny (±0.40)} & \underline{92.29} {\tiny (±0.35)} & \underline{82.32} {\tiny (±0.43)} & 89.26 {\tiny (±0.12)} & 88.01 {\tiny (±0.08)} & 73.07 {\tiny (±0.47)} & 69.60 {\tiny (±0.20)} \\
    w/ SC-SSL (Ours) & \underline{93.79} {\tiny (±0.22)} & \textbf{86.45} {\tiny (±0.02)} & \textbf{93.33} {\tiny (±0.19)} & \textbf{83.11} {\tiny (±0.17)} & \textbf{90.84} {\tiny (±0.29)} & \textbf{91.25} {\tiny (±0.34)} & \textbf{79.26} {\tiny (±0.31)} & \textbf{77.11} {\tiny (±0.35)} \\
    \bottomrule
  \end{tabular}
    \caption{Test accuracy of recent CISSL algorithms and our proposed SC-SSL under uniform, Gaussian, and unknown distributions on CIFAR10-LT and STL10-LT benchmarks.}
  \label{tab:2}
\end{table*}

\subsection{The Loss Design}
Building on the FixMatch algorithm \cite{Sohn2020FixMatchSS}, the supervised loss with balanced softmax for each classifier can be formulated as:
\begin{equation}
\mathcal{L}_{sup}(\tau, F) = \frac{1}{B_{l}} \sum_{i=1}^{B_{l}} \ell_{CE}\left( F\big(B({x}_{i}^{l})\big) + \tau \cdot \Delta p_{y_i}, y_{i} \right),
\end{equation}
where \(\Delta p_{y_i}\) is interpreted as the label frequency of class \(y_i\), \(\tau\) is a hyperparameter and \(\ell_{CE}\) denotes the cross entropy loss.
For unlabeled data, the consistency loss is:
\begin{equation}
\begin{gathered}
\mathcal{L}_{con}(\rho, F) = \frac{1}{B_{u}} \sum_{j=1}^{B_{u}} \ell_{CE}\left( F\big(B(\mathcal{A}(x_{j}^{u}))\big) , \Tilde{y}_{j} \right) \cdot \mathcal{M}(x_{j}^{u}, \rho), \\
\mathcal{M}(x_{j}^{u}, \rho)=\mathbb{I}\left(\operatorname*{max}\left(F(B(x_{j}^{u}))\right)\geq\rho\right),
\end{gathered}
\end{equation}
where \(\Tilde{y}_{j}\) denotes the predicted pseudo-labels and \(\rho\) is initialized from an estimated distribution, which is dynamically adjusted at each iteration \(t\) by:
\begin{equation}
    \rho^t(k) = \rho^{t-1}(k) - \alpha \cdot \mathbb{I}(\mathbf{b}_{opt}(k) > \nu),
\label{eq:dy}
\end{equation}
where \(\alpha\) and \(\nu\) are hyperparameters. \(\rho^0\) is initialized based on the expansion factor \(c\).

The assumed expansion factor \(c\) is determined by \(o^*\). Generally, the greater the relative quantity of non-head classes in the unlabeled data, the larger the expansion factor \(c\). Keeping \(\rho_b(\text{head})\) and \(\rho_e(\text{head})\) constant, we only need to initialize \(\rho_b(\text{non-head})\) and \(\rho_e(\text{non-head})\) based on \(c\):
\begin{align}
    \rho_b^0(\text{non-head}) &= \rho_{\text{max}} - \frac{c - 4}{10} \cdot min\{\gamma_u/50, 1\}, \\
    \rho_e^0(\text{non-head}) &= \rho_{\text{max}} - \frac{c - 3}{5} \cdot min\{\gamma_u/20, 1\}.
\end{align}
where \(\rho_\text{max}\) is a constant representing the maximum confidence level and \(\gamma_u\) is the imbalance ratio according to \(o^*\).

The overall training objective function is formulated as:
\begin{equation}
\begin{aligned}
    \mathcal{L} &= \mathcal{L}_{basic} \\
                &+ \mathcal{L}_{sup}^{b}(\tau_b, F_b) + \mathcal{L}_{con}^{b}(\rho_b, F_b) \\
                &+ \mathcal{L}_{sup}^{e}(\tau_e, F_e) + \mathcal{L}_{con}^{e}(\rho_e, F_e),
\end{aligned}
\end{equation}
where \(\mathcal{L}_{basic}\) is the loss of base SSL algorithms, \(\tau_b\) and \(\tau_e\) are hyperparameters for output classifier and expansive classifier respectively.

\section{Experiments}
\label{sec:exp}

\subsection{Analysis}

\subsection{Experimental Settings}
\paragraph{Datasets and Baselines} We conduct the experiments on widely used datasets, including CIFAR10-LT \cite{Krizhevsky2009LearningML}, CIFAR100-LT \cite{Krizhevsky2009LearningML}, STL10-LT \cite{Coates2011AnAO}, and ImageNet-127 \cite{Fan2021CoSSLCO}. 
Following previous works, we split the training data with a certain number of samples while controlling the imbalance ratio of data. \(tau_b\) and \(\tau_e\) are set to 2 and 4 for all datasets. For simplicity, we assume that \(c\) is dataset-independent and varies only with changes in the distribution of unlabeled data. Given 5 anchor distributions (consist, uniform, inverse, Gaussian, and inverse-Gaussian) above, the values of c are set to 4, 5, 6, 4, and 6, respectively. \(\rho_\text{max}\) is set to 0.95.  \(\rho_b(\text{head})\) and \(\rho_e(\text{head})\) are set to \(\rho_\text{max}\). Step size \(\alpha\) and threshold \(\nu\) in \cref{eq:dy} are set to 0.005 and 1.0 for all datasets.

We conduct comparisons with several CISSL algorithms that have been published in leading conferences and journals over the last few years. The baseline algorithms include LA \cite{Menon2020LongtailLV}, DARP \cite{Kim2020DistributionAR}, CReST \cite{Wei2021CReSTAC}, ABC \cite{Lee2021ABCAB}, DASO \cite{Oh2021DASODS}, ACR \cite{Wei2023TowardsRL}, CDMAD \cite{Lee2024CDMADCD}, CPE \cite{Ma2023ThreeHA}, and SimPro \cite{Du2024SimProAS}. All the reproduced results are consistent with the original paper or the codebase. We test our SC-SSL algorithm on the widely used USB codebase. More comparisons can be found in the Appendix.

\begin{table}[!t]
    \centering
    \small
    \begin{tabular}{>{\centering\arraybackslash}p{3.5cm}|>{\centering\arraybackslash}p{1.5cm}|>{\centering\arraybackslash}p{1.5cm}} 
        \toprule
        Algorithm / Resolution & \(32 \times 32\) &  \(64 \times 64\)  \\
        \midrule
        FixMatch  & 35.7 & 44.0 \\
        w/ CDMAD  & 55.6 & 61.7 \\
        w/ ACR & 57.2 & 63.6 \\
        w/ CPE & 57.8 & 64.1 \\
        w/ SimPro & 59.4 & 67.2 \\
        w/ SC-SSL & \textbf{62.3} & \textbf{69.4} \\
        \bottomrule
    \end{tabular}
        \caption{Test accuracy of recent CISSL algorithms and our proposed SC-SSL on ImageNet-127 benchmarks.}
    \label{tab:3}
\end{table}

\begin{table}[!t]
    \centering
    \small
    \begin{minipage}{0.37\linewidth}
        \centering
        \begin{tabular}{c|>{\centering\arraybackslash}p{0.6cm}|>{\centering\arraybackslash}p{0.6cm}}
            \toprule
            -  & C& I  \\
            \midrule
            2 & 83.89 & 86.02 \\
            3 & 83.54 & 85.98 \\
            4 & 83.50 & 86.15 \\
            \bottomrule
        \end{tabular}
        \caption{Comparison of different partitioning strategies. Classes are divided into 2/3/4 intervals. 'C' and 'I' denote consist and inverse, respectively.}
        \label{tab:whymedian}
    \end{minipage}%
    \hspace{0.05\linewidth} 
    \begin{minipage}{0.57\linewidth}
        \centering
        \begin{tabular}{>{\centering\arraybackslash}p{0.5cm}|>{\centering\arraybackslash}p{0.4cm}|>{\centering\arraybackslash}p{0.4cm}|>{\centering\arraybackslash}p{0.4cm}|>{\centering\arraybackslash}p{0.4cm}}
            \toprule
            -  & A1 & A2  & B1 & B2 \\
            \midrule
            ACR & 82.9 & 82.7  &  83.4 & 83.1 \\
            CPE & 82.0 & 81.4 &  82.7& 81.9\\
            Ours & \textbf{84.2} & \textbf{83.1} & \textbf{85.1} & \textbf{85.2}\\
            \bottomrule
        \end{tabular}
                \caption{Test accuracy on 4 chaotic distributions (\(N:M = 500:4000\)). A and B indicate the dominance of head/non-head classes in the unlabeled data. We generate 2 types of random distribution for each of them.}
        \label{tab:chaotic}
    \end{minipage}
\end{table}

\paragraph{Training Details}
We adhere to the default hyperparameters specified in USB, setting the batch sizes for labeled and unlabeled data to 64 and 128, respectively. For CIFAR10-LT, CIFAR100-LT, and STL10-LT, the input images are resized to \(32 \times 32\) pixels and the backbone network is WRN-28-2 \cite{Zagoruyko2016WideRN} without any pre-training.
For ImageNet-127, we downsample the images to \(64 \times 64\) and \(32 \times 32\) following CoSSL \cite{Fan2021CoSSLCO} and the backbone is ResNet-50 \cite{He2015DeepRL}. For the balanced and expansive classifier, the loss ratio of the unlabeled part is set to 2. 
Additionally, since the distributions of unlabeled data for comparison follow predefined distributions (as in some previous works \cite{Ma2023ThreeHA, Wei2023TowardsRL, Lee2024CDMADCD}), SC-SSL reweights \cite{Lai2022SmoothedAW} the loss function based on the prior of unlabeled sample counts.
We employ the SGD optimizer with a fixed learning rate of 0.03, a momentum of 0.9, and a weight decay of 0.0005. 

\subsection{Main Results}
In our experiments, we assessed the test accuracy of various CISSL algorithms, including our proposed SC-SSL, on CIFAR10-LT and CIFAR100-LT datasets. Using a WRN-28-2 architecture trained from scratch, we measured performance under consistent and inverse distributions as well as uniform and Gaussian distributions. The results are shown in \cref{tab:1}, \cref{tab:2} and \cref{tab:3}. 


\paragraph{Why are classes simply divided into head and non-head types? —— What if the unlabeled data distribution is chaotic?} 
Due to the current research focusing on long-tailed distributions, there is a significant disparity in the amount of labeled data between head and tail classes. Our goal is to enable non-head classes to leverage unlabeled data to enhance generalization performance, so we have roughly categorized them into two groups. This distinction is sufficient for the distributions discussed in previous work. Our experimental results in \cref{tab:whymedian} also support this when we try finer-grained \(\rho\) initialization. When the unlabeled data is chaotic, as the example in \cref{fig:2} shows, there is no significant difference in quantity between non-head and head classes. We can treat them as uniformly distributed, assuming \(c=5\), since the dynamic sampling probabilities will come into play. The experimental results are shown in \cref{tab:chaotic}.



\begin{table}[t]
  \centering
  \small
    \begin{tabular}{>{\centering\arraybackslash}p{1.0cm}>{\centering\arraybackslash}p{0.7cm}>{\centering\arraybackslash}p{0.7cm}|>{\centering\arraybackslash}p{1.2cm}|>{\centering\arraybackslash}p{1.2cm}}
    \toprule
    \(\rho\) & \(\tau\) & \(c\) & Consist & Inverse \\
    \midrule
     c-init & (0, 0) & 4 &  82.33& 82.16\\
     c-init & (2, 2) & 6 &  81.55 & 86.79\\
     c-init & (1, 2) & 4 &  85.28 & 86.12\\
     c-init & (2, 4) & 4 &  \textbf{86.53} & 88.54\\
     c-init & (2, 4) & 6 & 84.11 & \textbf{89.97} \\
     fix-max & (2, 4) &  - &  84.70 & 86.90\\
     fix-max &  (2, \text{-} ) & - & 83.80 & 86.49 \\
     fix-min &  (2, 1) & - &  80.22 & 85.65\\
     fix-min &  (2, 4) & - &  79.04 & 86.05\\
    \bottomrule
    \end{tabular}
      \caption{Ablation study of the key factors in our framework. ‘fix’ indicates the \(\rho\) is set to be \(\rho_{max}\) or \(\rho_{min}\) without dynamic sampling.}
    \label{tab:factors}
\end{table}%

\paragraph{What is the impact of ignoring certain factors?}
All three key factors are essential. Without \(\gamma\), it is impossible to initialize the sampling settings; when \(c\) is small, excessively high sampling probabilities can lead to confirmation errors. If \(\Delta p\) is not used, it becomes challenging to address the impact of bias in the labeled data itself. Furthermore, without \(\rho\), there is no way to dynamically adjust the sampling probabilities, as \(\Delta p\) is determined based on the labeled data, making it difficult to estimate the effect of unlabeled data consistency regularization on it. The experimental results (\(N:M = 1500:3000\)) are shown in \cref{tab:factors}.


\bibliography{aaai2026}

\newpage
\makeatletter
\@ifundefined{isChecklistMainFile}{
  \newif\ifreproStandalone
  \reproStandalonetrue
}{
  \newif\ifreproStandalone
  \reproStandalonefalse
}
\makeatother

\ifreproStandalone
\documentclass[letterpaper]{article}
\usepackage[submission]{aaai2026}
\setlength{\pdfpagewidth}{8.5in}
\setlength{\pdfpageheight}{11in}
\usepackage{times}
\usepackage{helvet}
\usepackage{courier}
\usepackage{xcolor}
\usepackage{times}  
\usepackage{helvet}  
\usepackage{courier}  
\usepackage[hyphens]{url}  
\usepackage{graphicx} 
\urlstyle{rm} 
\def\UrlFont{\rm}  
\usepackage{natbib}  
\usepackage{caption} 
\frenchspacing  
\setlength{\pdfpagewidth}{8.5in} 
\setlength{\pdfpageheight}{11in} 
%
\usepackage{algorithm}

%
\usepackage{newfloat}
\usepackage{listings}
\usepackage{array}
\usepackage[noend]{algpseudocode}
\usepackage{amsmath}
\usepackage{booktabs} 
\usepackage{amsthm}
\usepackage{amssymb}
\usepackage{cleveref}
\newtheorem{theorem}{Theorem}[section]
\newtheorem*{theorem*}{Theorem}

\frenchspacing

\begin{document}
\fi
\setlength{\leftmargini}{20pt}
\makeatletter\def\@listi{\leftmargin\leftmargini \topsep .5em \parsep .5em \itemsep .5em}
\def\@listii{\leftmargin\leftmarginii \labelwidth\leftmarginii \advance\labelwidth-\labelsep \topsep .4em \parsep .4em \itemsep .4em}
\def\@listiii{\leftmargin\leftmarginiii \labelwidth\leftmarginiii \advance\labelwidth-\labelsep \topsep .4em \parsep .4em \itemsep .4em}\makeatother

\setcounter{secnumdepth}{0}
\renewcommand\thesubsection{\arabic{subsection}}
\renewcommand\labelenumi{\thesubsection.\arabic{enumi}}

\section*{Supplementary Material}

\vspace{1em}
\hrule
\vspace{1em}








\section{Proof of \textbf{Theorem 0.1}}
\label{sec:proof}

\begin{theorem*}
The pseudo-label \(Y_{psl}\) has the following probability distribution:
\begin{equation*}
\begin{aligned}
P(Y_{psl}=1) &= \gamma \Phi \left( \frac{\frac{\mu_{2}-\mu_{1}}{2} - \frac{1}{\beta}\log\left(\frac{\rho}{1-\rho}\right) - \Delta p}{\sigma_{2}} \right) \\
&\quad + (1-\gamma)\Phi\left( \frac{\frac{\mu_{1}-\mu_{2}}{2} - \frac{1}{\beta}\log\left(\frac{\rho}{1-\rho}\right) - \Delta p}{\sigma_{1}} \right), \\
P(Y_{psl}=-1) &= (1-\gamma)\Phi\left( \frac{\frac{\mu_{2}-\mu_{1}}{2} - \frac{1}{\beta}\log\left(\frac{\rho}{1-\rho}\right) + \Delta p}{\sigma_{1}} \right) \\
&\quad + \gamma\Phi\left( \frac{\frac{\mu_{1}-\mu_{2}}{2} - \frac{1}{\beta}\log\left(\frac{\rho}{1-\rho}\right) + \Delta p}{\sigma_{2}} \right), \\
P(Y_{psl}=0) &= 1 - P(Y_{psl}=1) - P(Y_{psl}=-1),
\end{aligned}
\end{equation*}
where $\Phi(\cdot)$ denotes the cumulative distribution function (CDF) of the standard normal distribution.
\end{theorem*}

\begin{proof}
We consider a binary classification setting with class prior $P(Y = 1) = \gamma$ and $P(Y = -1) = 1 - \gamma$. The input $\mathcal{x}$ follows class-conditional distributions:
\[
\mathcal{x} \mid Y = 1 \sim \mathcal{N}(\mu_2, \sigma_2^2), \quad \mathcal{x} \mid Y = -1 \sim \mathcal{N}(\mu_1, \sigma_1^2),
\]
where $\mu_2 > \mu_1$ without loss of generality. The classifier predicts a confidence score:
\[
s(\mathcal{x}) = \frac{1}{1 + \exp\left( -\beta\left( \mathcal{x} - \Delta p - \frac{\mu_1 + \mu_2}{2} \right) \right)},
\]
where $\Delta p$ is the logit adjustment term and $\beta$ reflects the model's confidence sharpness. A pseudo-label $Y_{psl}$ is assigned by comparing $s(\mathcal{x})$ to a fixed threshold $\rho \in (0.5, 1)$:

\begin{itemize}
    \item $Y_{psl} = 1$ if $s(\mathcal{x}) > \rho$,
    \item $Y_{psl} = -1$ if $s(\mathcal{x}) < 1 - \rho$,
    \item $Y_{psl} = 0$ otherwise (i.e., abstain).
\end{itemize}

We now compute $P(Y_{psl}=1)$ by considering the threshold condition for $s(\mathcal{x}) > \rho$. This is equivalent to:
\[
\mathcal{x} > \frac{\mu_1 + \mu_2}{2} + \frac{1}{\beta} \log\left( \frac{\rho}{1 - \rho} \right) + \Delta p.
\]
Define the right-hand side as $T$:
\[
T := \frac{\mu_1 + \mu_2}{2} + \frac{1}{\beta} \log\left( \frac{\rho}{1 - \rho} \right) + \Delta p.
\]

Now, under the distribution $\mathcal{x} \mid Y=1 \sim \mathcal{N}(\mu_2, \sigma_2^2)$:
\begin{align*}
P(Y_{psl}=1 \mid Y=1) &= P(\mathcal{x} > T \mid Y=1) \\
&= 1 - \Phi\left( \frac{T - \mu_2}{\sigma_2} \right) \\
&= \Phi\left( \frac{\mu_2 - T}{\sigma_2} \right).
\end{align*}

Substituting $T$ back in:
\[
P(Y_{psl}=1 \mid Y=1) = \Phi\left( \frac{\frac{\mu_2 - \mu_1}{2} - \frac{1}{\beta} \log\left( \frac{\rho}{1 - \rho} \right) - \Delta p}{\sigma_2} \right).
\]

Similarly, under $Y=-1$:
\[
P(Y_{psl}=1 \mid Y=-1) = \Phi\left( \frac{\frac{\mu_1 - \mu_2}{2} - \frac{1}{\beta} \log\left( \frac{\rho}{1 - \rho} \right) - \Delta p}{\sigma_1} \right).
\]

Thus, using the law of total probability:
\begin{align*}
P(Y_{psl}=1) &= \gamma \cdot P(Y_{psl}=1 \mid Y=1) \\
&+ (1 - \gamma) \cdot P(Y_{psl}=1 \mid Y=-1) \\
&= \gamma \Phi\left( \frac{\frac{\mu_2 - \mu_1}{2} - \frac{1}{\beta} \log\left( \frac{\rho}{1 - \rho} \right) - \Delta p}{\sigma_2} \right) \\
&\quad + (1 - \gamma) \Phi\left( \frac{\frac{\mu_1 - \mu_2}{2} - \frac{1}{\beta} \log\left( \frac{\rho}{1 - \rho} \right) - \Delta p}{\sigma_1} \right).
\end{align*}

The derivation of $P(Y_{psl} = -1)$ follows the same logic, with threshold:
\[
\mathcal{x} < \frac{\mu_1 + \mu_2}{2} - \frac{1}{\beta} \log\left( \frac{\rho}{1 - \rho} \right) + \Delta p,
\]
leading to the expressions in the theorem.

Finally, $P(Y_{psl}=0)$ is obtained by subtracting:
\[
P(Y_{psl}=0) = 1 - P(Y_{psl}=1) - P(Y_{psl}=-1).
\]

This completes the proof.
\end{proof}

\section{Overview of Algorithm~\ref{alg:scssl}.}
The SC-SSL algorithm presents a training framework that integrates adaptive sampling control into semi-supervised learning under class-imbalanced settings. It maintains two classifiers---a standard output classifier $F_b$ and an expansive classifier $F_e$---that share a common backbone $B(\cdot)$. At the beginning of training, an anchor distribution matching procedure is used to estimate the expansion factor $c$, which guides the initialization of class-wise confidence thresholds $\rho_b$ and $\rho_e$.

In each iteration, supervised loss is computed using balanced softmax adjustment for labeled samples. For unlabeled data, SC-SSL generates pseudo-labels via the expansive classifier and applies a dynamic thresholding strategy to decide which predictions are confident enough to be used for consistency regularization. The key novelty lies in adaptively lowering the thresholds for non-head classes based on an optimization bias signal $\mathbf{b}_{opt}$, thereby increasing their sampling probability.

Finally, both classifiers are jointly updated by minimizing the total loss comprising supervised and consistency components. This design allows SC-SSL to expand the decision boundaries of minority classes while preserving classification robustness, thereby addressing the dual challenges of label imbalance and pseudo-label noise.

\begin{algorithm}[!t]
\caption{SC-SSL}
\label{alg:scssl}
\begin{algorithmic}[1]
\Require Labeled dataset $\mathcal{D}_l = \{(x_i^l, y_i^l)\}$, Unlabeled dataset $\mathcal{D}_u = \{x_j^u\}$
\Require Backbone $B(\cdot)$, Classifiers $F_b, F_e$, Confidence thresholds $\rho_b, \rho_e$
\Require Hyperparameters $\tau_b, \tau_e, \alpha, \nu, \rho_{\max}$
\State Initialize expansion factor $c$ via anchor distribution matching (KL)
\State Initialize $\rho^0_b(\text{non-head}), \rho^0_e(\text{non-head})$ based on $c$

\For{each training step $t = 1, \ldots, T$}
    \State Sample minibatches $\mathcal{B}_l \subseteq \mathcal{D}_l$, $\mathcal{B}_u \subseteq \mathcal{D}_u$
    \ForAll{$(x_i^l, y_i^l) \in \mathcal{B}_l$}
        \State Compute features $h_i \gets B(x_i^l)$
        \State Compute supervised loss:
        \[\mathcal{L}_{sup}(x_i^l) = \ell_{CE}(F(h_i) + \tau \cdot \Delta p_{y_i}, y_i^l)\]
    \EndFor

    \ForAll{$x_j^u \in \mathcal{B}_u$}
        \State Get weak and strong augmentations: $x_j^u$, $\mathcal{A}(x_j^u)$
        \State Compute pseudo-label:
        \[\tilde{y}_j \gets \arg\max \sigma(F_e(B(x_j^u)))\]
        \State Compute confidence: $s_j \gets \max \sigma(F_e(B(x_j^u)))$
        \If{$s_j > \rho^t_e(\tilde{y}_j)$}
            \State Compute consistency loss:
            \[\mathcal{L}_{con}(x_j^u) = \ell_{CE}(F_e(B(\mathcal{A}(x_j^u))), \tilde{y}_j)\]
        \EndIf
    \EndFor

    \State Update $\rho^{t+1}_e(k) \leftarrow \rho^t_e(k) - \alpha \cdot \mathbb{I}(\mathbf{b}_{opt}(k) > \nu)$
    \State Same update for $\rho^{t+1}_b(k)$ if applied to $F_b$
    \State Update parameters via total loss:
    \[
        \mathcal{L} = \mathcal{L}_{basic} + \mathcal{L}_{sup}^b + \mathcal{L}_{con}^b + \mathcal{L}_{sup}^e + \mathcal{L}_{con}^e
    \]
\EndFor
\State \Return Trained classifiers $F_b$, $F_e$
\end{algorithmic}
\end{algorithm}

\begin{table*}[!t]
  \centering
  \caption{Test accuracy on Cifar100 (labeld nums 150, unlabeled nums 300) and STL10 (labeled 150).}
  \begin{tabular}{@{}c|cc|cc@{}}
    \toprule
    Dataset & \multicolumn{2}{c|}{Cifar100} & \multicolumn{2}{c}{STL10}  \\
    \(\gamma_l \text{:} \gamma_u\) &10:10 & 20:20 & 10: & 20: \\
    \midrule
    FixMatch & 47.92 {\tiny (±0.55)} & 42.23 {\tiny (±0.89)} & 66.67 {\tiny (±1.89)} & 55.99 {\tiny (±3.84)} \\
    w/ ADELLO [3] & 58.72 {\tiny (±0.18)} & NA & 75.70 {\tiny (±0.70)} & 74.60 {\tiny (±0.40)} \\
    w/ RECD [5]  & 59.41 {\tiny (±0.37)} & NA & 76.01 {\tiny (±0.55)} & 74.95 {\tiny (±0.66)}  \\
    w/ TCBC [6] & 59.40 {\tiny (±0.28)} & 53.90 {\tiny (±0.72)} & 77.60 {\tiny (±0.93)} & 74.90 {\tiny (±1.42)} \\
    w/ IFMatch [4] & NA & 53.61 {\tiny (±0.25)} & NA & NA  \\
    w/ BEM+ACR & 60.99 {\tiny (±0.55)} & 59.80 {\tiny (±0.37)} & 79.30 {\tiny (±0.34)} & 75.90 {\tiny (±0.15)}  \\
    w/ BEM[1] & 59.00 {\tiny (±0.23)} & 54.30 {\tiny (±0.36)} & 74.32 {\tiny (±1.05)} & 72.69 {\tiny (±0.88)} \\
    w/ BEM+Ours & \textbf{61.88} {\tiny (±0.23)} & \textbf{60.69} {\tiny (±0.18)} & \textbf{80.70} {\tiny (±0.41)} & \textbf{78.05} {\tiny (±0.35)} \\
    w/ Ours & \textbf{60.65} {\tiny (±0.23)} & \textbf{59.88} {\tiny (±0.37)} & \textbf{79.26} {\tiny (±0.31)} & \textbf{77.11} {\tiny (±0.55)} \\
    \bottomrule
  \end{tabular}
  \label{tab:reb1}
\end{table*}

\begin{table*}[!t]
  \centering
  \caption{Test accuracy on Cifar10 (labeld 1500, unlabeled 3000) and ImageNet127. "-" denotes the reverse distribution. "x" denotes the resolution of images.}
  \begin{tabular}{@{}c|ccc|cc@{}}
    \toprule
    Dataset & \multicolumn{3}{c|}{Cifar10} & \multicolumn{2}{c}{ImageNet} \\
    \(\gamma_l \text{:} \gamma_u\) & 100:100 & 100:-100 & 150:150 & x32 & x64\\
    \midrule
    FixMatch & 76.49 {\tiny (±0.72)} & 68.92 {\tiny (±0.79)} & 72.15 {\tiny (±0.94)}  & 35.7 & 44.0 \\
    w/ ADELLO [3] & 83.80 {\tiny (±0.30)} & 86.10 {\tiny (±0.40)} & 79.47 {\tiny (±0.22)} & 47.5 & 58.0 \\
    w/ RECD [5] & 84.60 {\tiny (±0.13)} & 85.91 {\tiny (±0.46)} & 80.60 {\tiny (±0.53)} & 47.3 & 59.5 \\
    w/ TCBC [6] & 84.00 {\tiny (±0.55)} & 85.70 {\tiny (±0.17)} & 80.40 {\tiny (±0.58)}  & 48.2 & NA \\
    w/ IFMatch [4] & 79.46 {\tiny (±0.35)} & NA & 75.59 {\tiny (±0.48)} & 48.6 & NA \\
    w/ BEM+ACR & 85.50 {\tiny (±0.28)} & 89.80 {\tiny (±0.12)} & 83.80 {\tiny (±0.12)} & 58.0 & 63.9 \\
    w/ BEM[1] & 80.30 {\tiny (±0.62)} & 79.10 {\tiny (±0.77)} & 75.70 {\tiny (±0.22)} & 53.3 & 58.2 \\
    w/ BEM+Ours & \textbf{87.45} {\tiny (±0.11)} & \textbf{90.50} {\tiny (±0.25)} & \textbf{85.03} {\tiny (±0.12)}  & \textbf{62.9} & \textbf{69.7} \\
    w/ Ours & \textbf{86.53} {\tiny (±0.16)} & \textbf{89.97} {\tiny (±0.20)} & \textbf{83.86} {\tiny (±0.15)} & \textbf{62.3} & \textbf{69.4}\\
    \bottomrule
  \end{tabular}
  \label{tab:reb2}
\end{table*}

\begin{figure*}[!t]
    \centering
    \begin{minipage}{0.30\textwidth}
        \includegraphics[width=\linewidth]{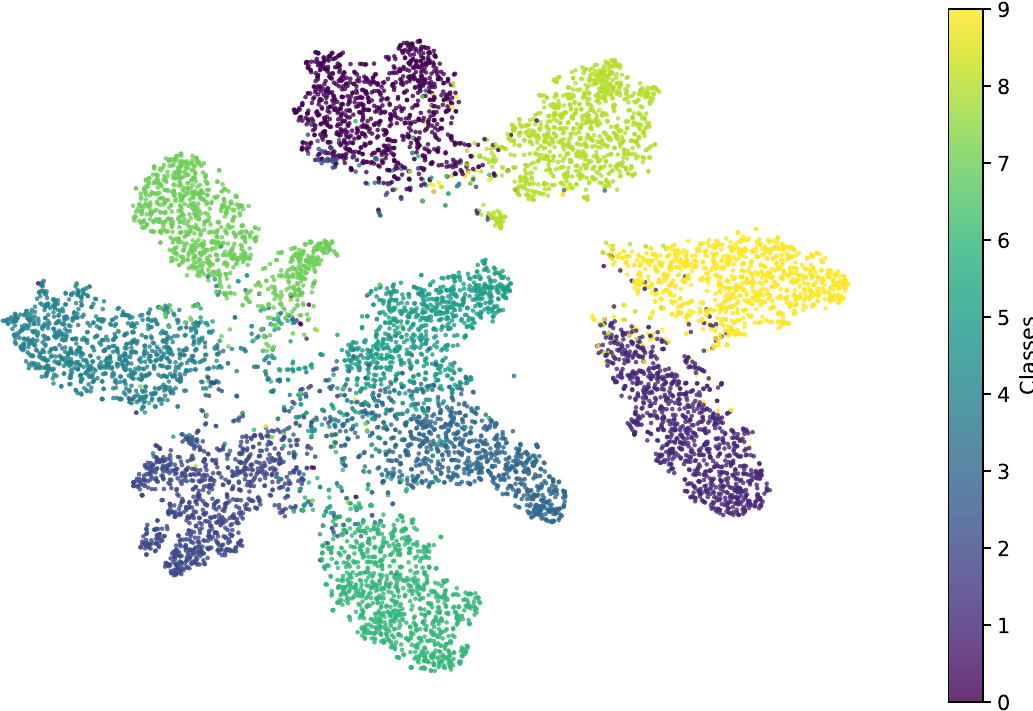}
    \end{minipage}%
    \hspace{0.1cm}
    \begin{minipage}{0.30\textwidth}
        \includegraphics[width=\linewidth]{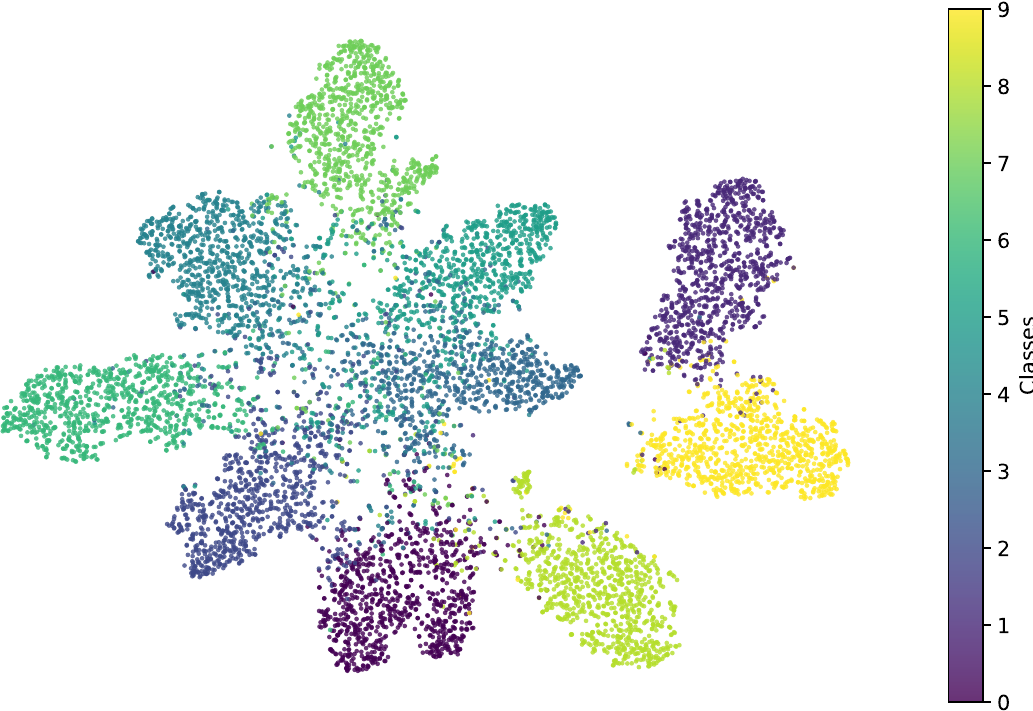}
    \end{minipage}%
    \hspace{0.1cm}
    \begin{minipage}{0.30\textwidth}
        \includegraphics[width=\linewidth]{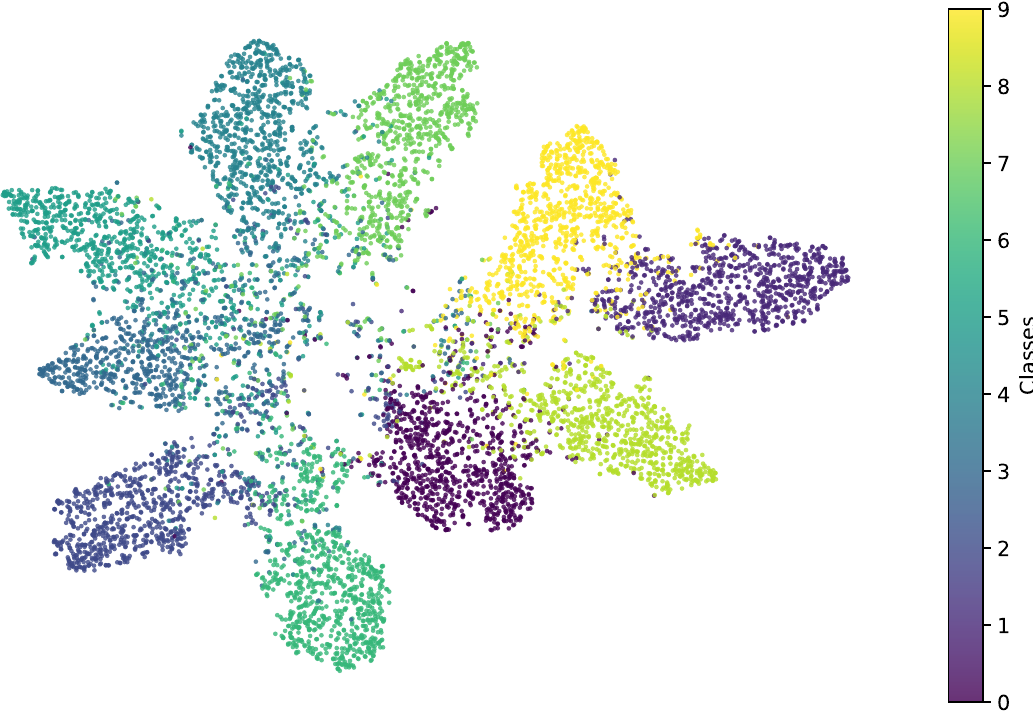}
    \end{minipage}


    \begin{minipage}{0.30\textwidth}
        \includegraphics[width=\linewidth]{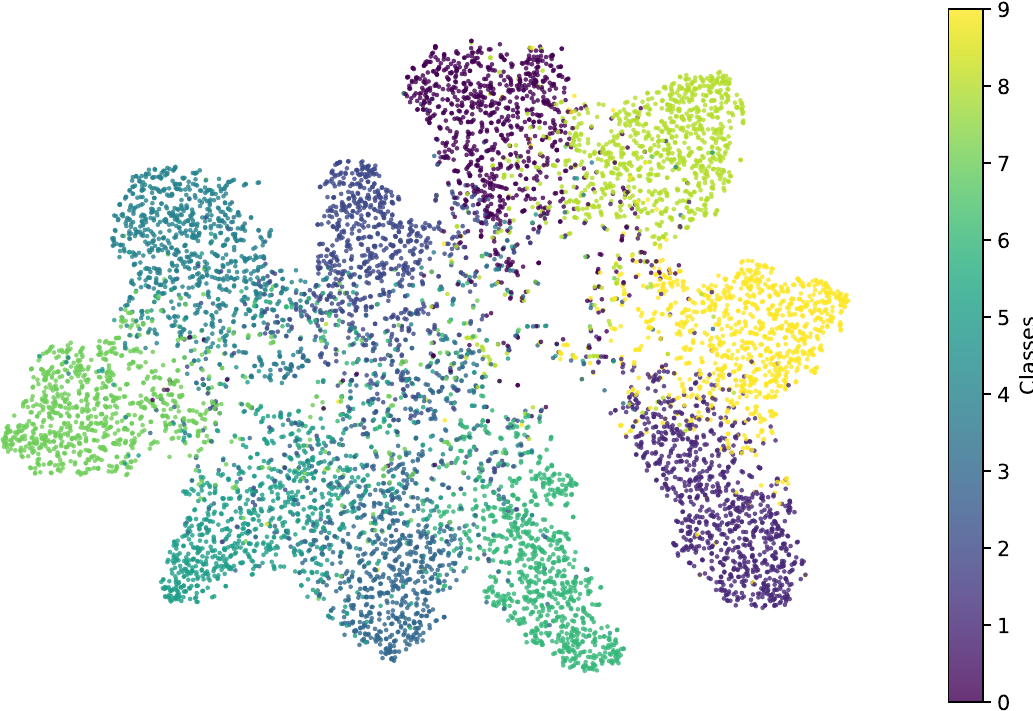}
        \caption*{Uniform}
    \end{minipage}%
    \hspace{0.1cm}
    \begin{minipage}{0.30\textwidth}
        \includegraphics[width=\linewidth]{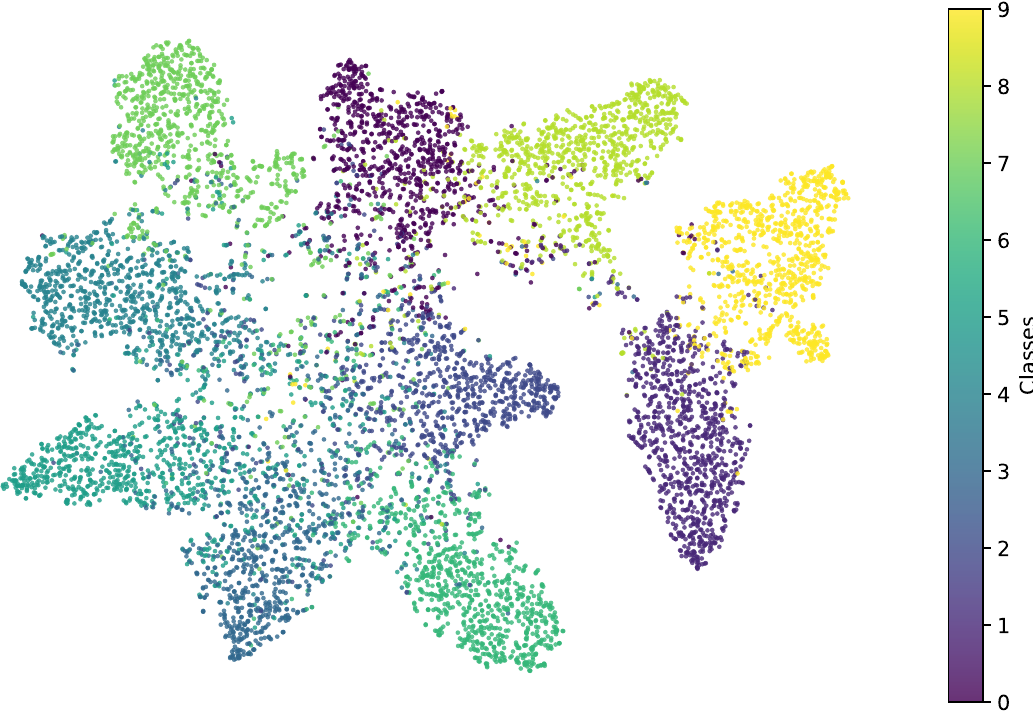}
        \caption*{Inverse}
    \end{minipage}%
    \hspace{0.1cm}
    \begin{minipage}{0.30\textwidth}
        \includegraphics[width=\linewidth]{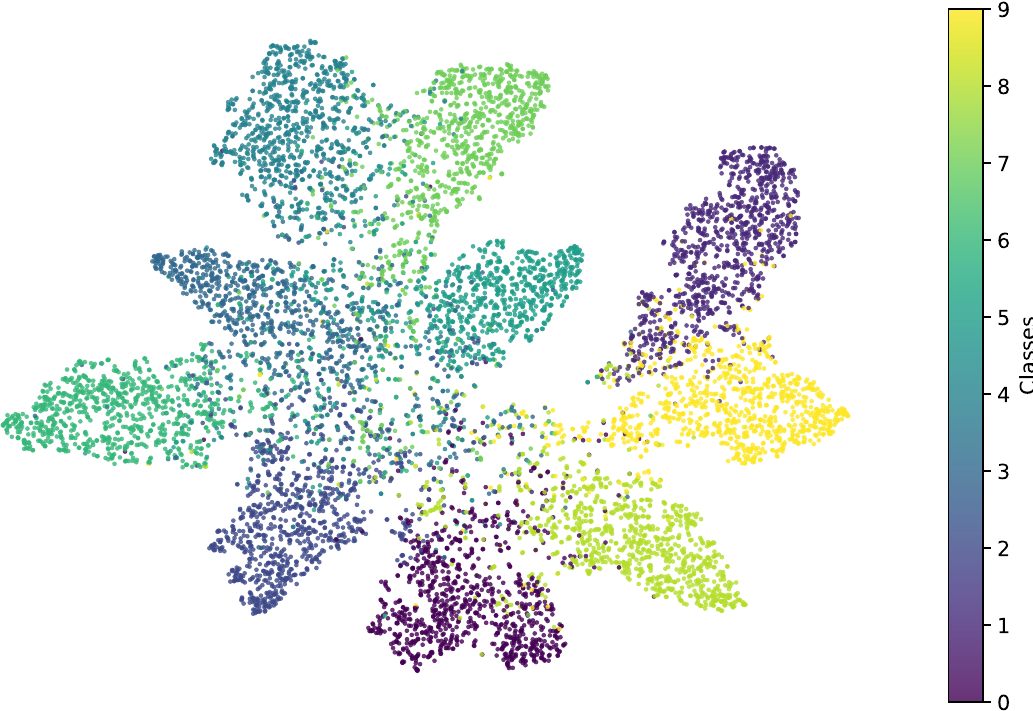}
        \caption*{Consist}
    \end{minipage}

    \caption{t-SNE visualization results on CIFAR10-LT(Top: SC-SSL, bottom: ABC).}
    \label{fig:tsne}
\end{figure*}

\begin{figure*}[!t]
    \centering
    \begin{minipage}{0.28\textwidth}
        \includegraphics[width=\linewidth]{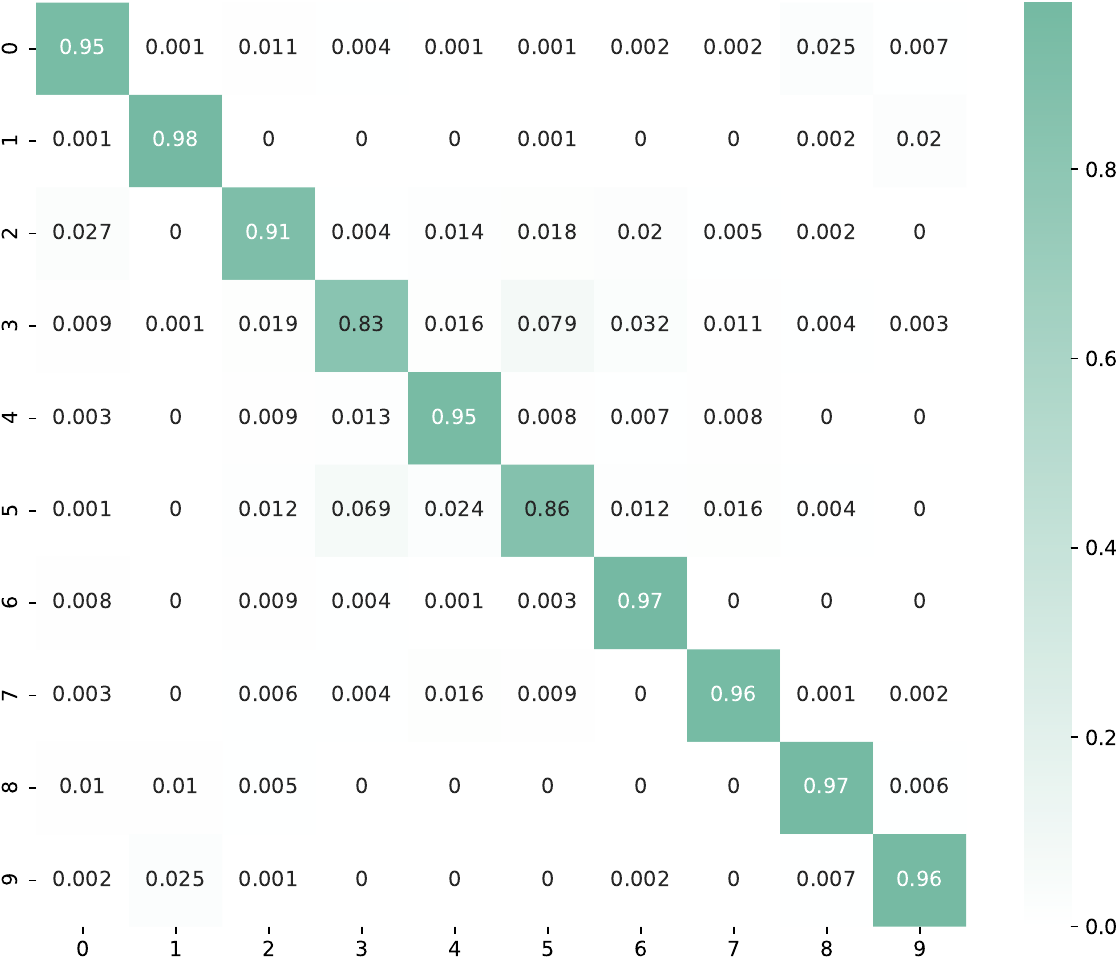}
        \caption*{Uniform}
    \end{minipage}%
    \hspace{0.2cm}
    \begin{minipage}{0.28\textwidth}
        \includegraphics[width=\linewidth]{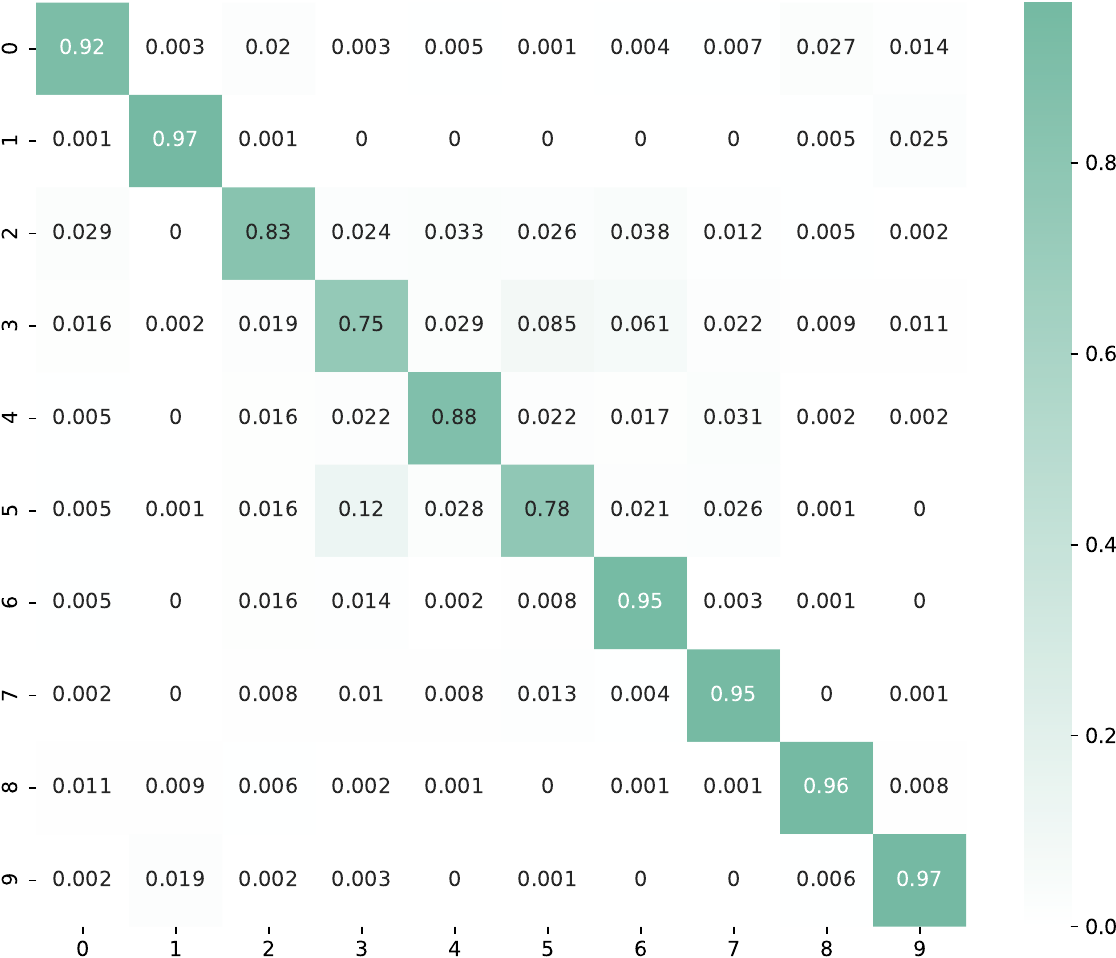}
        \caption*{Inverse}
    \end{minipage}%
    \hspace{0.2cm}
    \begin{minipage}{0.28\textwidth}
        \includegraphics[width=\linewidth]{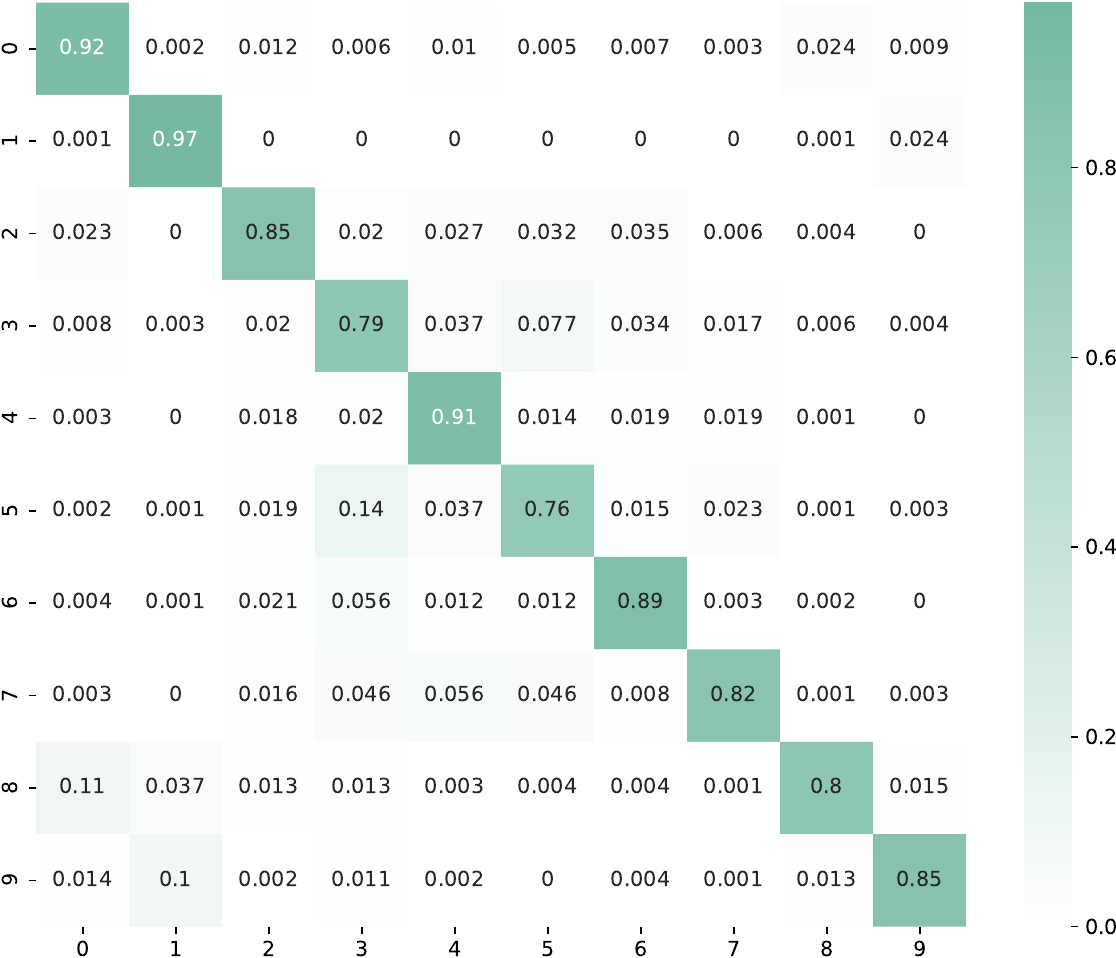}
        \caption*{Consist}
    \end{minipage}

    \caption{Confusion matrix on CIFAR10-LT test set.}
    \label{fig:confusion}
\end{figure*}

\begin{figure*}[!t]
    \centering
    \includegraphics[width=15cm]{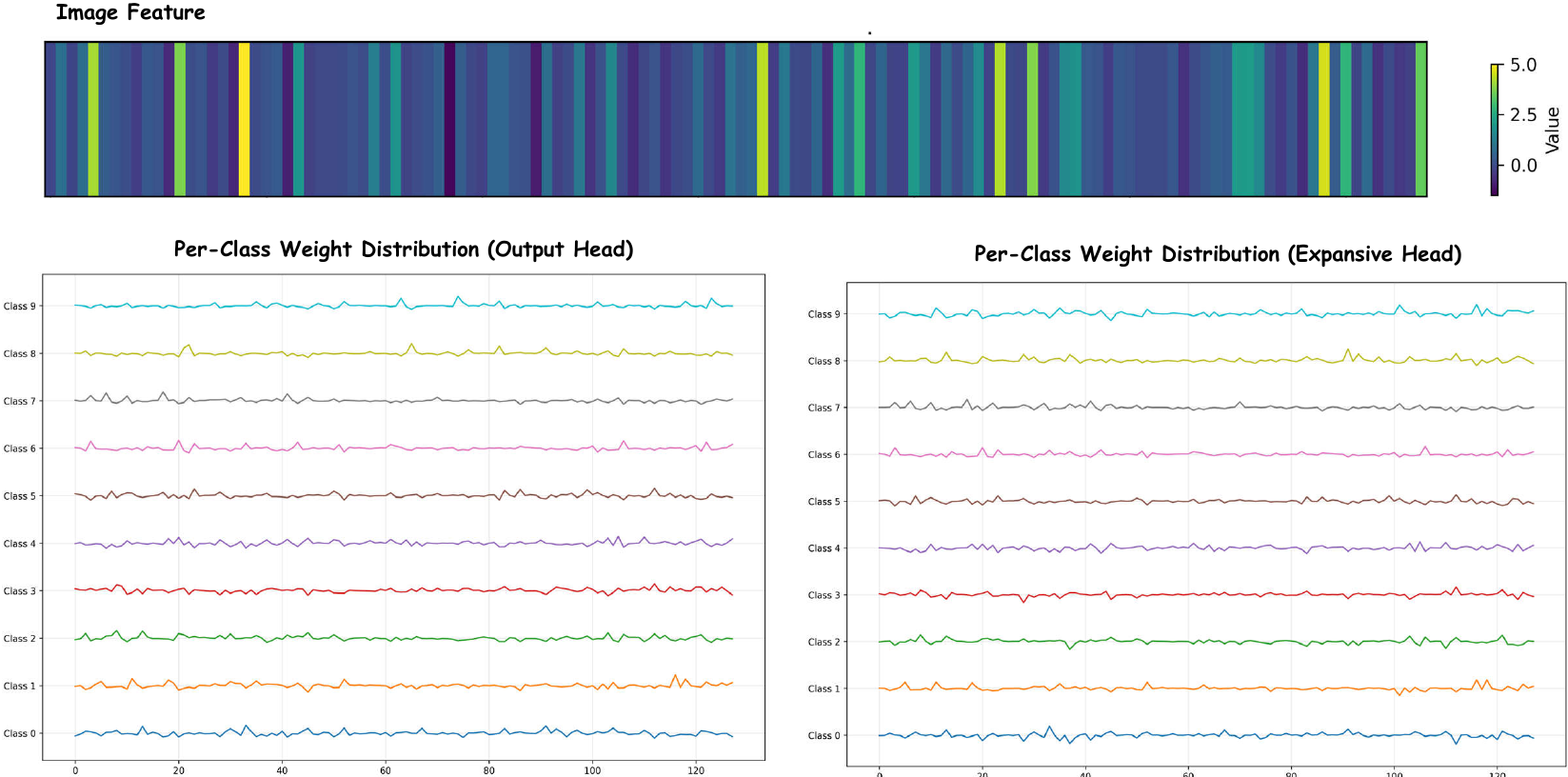}
    \caption{}
    \label{fig:vis_w}
    \vspace{-0.5cm}
\end{figure*}

\section{More Comparison with SOTA Methods}

Table~\ref{tab:reb1} and Table~\ref{tab:reb2} comprehensively compare our method with existing state-of-the-art approaches on four benchmark datasets: CIFAR100, STL10, CIFAR10, and ImageNet127, under various labeled-unlabeled ratios and distribution scenarios.

In Table~\ref{tab:reb1}, we evaluate the performance on CIFAR100 and STL10 with extremely limited labeled data (150 samples) and varying imbalance settings (\(\gamma_l : \gamma_u = 10:10\) and \(20:20\)). FixMatch yields suboptimal performance under these imbalanced conditions. While ADELLO, RECD, and TCBC provide modest improvements, our method—both standalone and combined with BEM—achieves the highest test accuracy across all configurations. Notably, the combination of BEM and our SC-SSL achieves a substantial improvement over BEM+ACR, with up to +2.1\% gain on CIFAR100 and +1.7\% gain on STL10, indicating that our sampling strategy better supports balanced feature learning under long-tailed semi-supervised scenarios.

In Table~\ref{tab:reb2}, we report results on CIFAR10 and ImageNet127. We consider different imbalance directions (e.g., \(\gamma_l : \gamma_u = 100:-100\)) and image resolutions (32×32 and 64×64). Our method consistently outperforms all baselines across these more realistic and challenging setups. Particularly on the large-scale ImageNet127 dataset, our approach improves the performance by a clear margin, achieving 62.9\% and 69.6\% accuracy on 32×32 and 64×64 resolutions respectively—substantially higher than prior methods such as ADELLO, RECD, and TCBC. These results demonstrate that our pseudo-label sampling control not only generalizes well across datasets, but also scales effectively with data complexity and image resolution.

Overall, these tables highlight the effectiveness, robustness, and scalability of our proposed method under both small-data and large-scale long-tailed semi-supervised learning settings.

\section{t-SNE results}

To further validate the effectiveness of SC-SSL, we visualize the learned feature space using t-SNE and report the confusion matrices for model predictions.

As shown in Fig.~\ref{fig:tsne}, the feature embeddings produced by our expansive classifier exhibit a significantly better separation between head and non-head classes compared to the baseline. Notably, samples from minority (non-head) classes are no longer collapsed into ambiguous regions, indicating that SC-SSL successfully mitigates feature dominance by head classes. The improved inter-class margin and intra-class compactness suggest that our sampling-aware feature learning promotes more discriminative and balanced representations across the long-tailed label spectrum.

\section{Confusion Matrics}

Additionally, the confusion matrices in Fig.~\ref{fig:confusion} quantitatively demonstrate the advantage of our approach. Compared to the baseline, our model yields noticeably higher recall for tail classes, as reflected by the stronger diagonal patterns even in the lower-frequency class regions. This aligns with our theoretical motivation that adaptive pseudo-label sampling (enabled by the expansive classifier and optimization bias correction) helps expand decision boundaries for under-represented classes, thereby enhancing robustness. 

Together, these visualizations confirm that SC-SSL not only improves overall classification performance but also achieves more balanced learning under imbalanced semi-supervised settings.

\section{Visualization of Linear Layer and Features}

Figure~\ref{fig:vis_w} presents a qualitative analysis of the learned representations and classifier parameters. The upper part shows the 128-dimensional features projected to 2D via t-SNE, indicating that the learned representations vary significantly across samples and are highly dependent on input augmentation and feature dynamics. The lower part visualizes the absolute magnitude of the linear classifier's weights across different classes. Although some classes exhibit larger weights than others, this pattern is influenced by the current feature distribution during forward propagation rather than the intrinsic data imbalance in the training set. In other words, the absolute weight values alone do not reliably reflect the class frequency or long-tailed structure of the dataset, since weight magnitudes adapt to the scale and dispersion of the input features. This motivates our use of the bias term as a more stable proxy for optimization-induced imbalance, independent of feature variation.


\ifreproStandalone
\end{document}
\fi
\end{document}